\documentclass{article}
\usepackage{matus}

\usepackage{tikz}
\usepackage{pgfplots}
\pgfplotsset{compat=1.11}
\usetikzlibrary{intersections,decorations.pathreplacing,positioning}

\DeclareMathOperator{\polylog}{polylog}
\DeclareMathOperator{\Proj}{Proj}
\DeclareMathOperator{\Range}{Range}

\title{Size-Noise Tradeoffs in Generative Networks}
\author{Bolton Bailey\qquad Matus Telgarsky\\
\tt{\{boltonb2,mjt\}@illinois.edu}\\
University of Illinois, Urbana-Champaign}
\date{}

\begin{document} 

\maketitle

\begin{abstract}
  This paper investigates the ability of generative networks to convert their input noise distributions into other distributions.
  Firstly, we demonstrate a construction that allows ReLU networks to increase the dimensionality of their noise distribution by implementing a ``space-filling'' function based on iterated tent maps. We show this construction is optimal by analyzing the number of affine pieces in functions computed by multivariate ReLU networks. 
  Secondly, we provide efficient ways (using $\polylog(1/\epsilon)$ nodes) for networks to pass between univariate uniform and normal distributions, using a Taylor series approximation and a binary search gadget for computing function inverses. Lastly, we indicate how high dimensional distributions can be efficiently transformed into low dimensional distributions.
\end{abstract}

\section{Introduction}

This paper focuses on the representational capabilities of generative networks. 
A generative network models a complex target distribution by taking samples from some efficiently-sampleable noise distribution and mapping them to the target distribution using a neural network.
What distributions can a generative net approximate, and how well?
Larger neural networks or networks with more noise given as input have greater power to model distributions, but it is unclear how the use of one resource can make up for the lack of the other.
We seek to describe the relationship between these resources. 

In our analysis, we make a few assumptions on the structure of the network and the noise.
We focus on the two most standard choices for noise distributions: The normal distribution, and the uniform distribution on the unit hypercube \citep{arjovsky2017wasserstein}. Henceforth, we will use the term ``uniform distribution'' to refer to the uniform distribution on unit hypercubes, unless otherwise specified. We look specifically at the case where the generative network is a fully-connected network with the ReLU activation function (without weight sharing). The notion of approximation we use is the \emph{Wasserstein distance}, introduced for generative networks in \citet{arjovsky2017wasserstein}, which is defined as follows:
\begin{definition}
    For two distributions $\mu$ and $\nu$ on $\R^d$, their Wasserstein distance is defined as
    \[ 
        W(\mu, \nu) := \inf_{\pi \in \Pi(\mu, \nu)} \int |x-y| d\pi(x, y),  
    \]
    where $\Pi(\mu, \nu)$ is the set of joint distributions having $\mu$ and $\nu$ as marginals.
\end{definition}

Our results fall into three regimes, each covered in its own section: 
\begin{description}
    \item[Section \ref{nld}: The case where the input dimension is less than the output dimension.] 
    \hfill \\    
    In this regime, we prove tight upper and lower bounds for the task of approximating higher dimensional uniform distributions with lower dimensional distributions in terms of the average width $W$ and depth (number of layers) $L$ of the network.
    The bounds are tight in the sense that both give an accuracy of the form $\eps = O(W)^{-O(L)}$ (keeping input and output dimensions fixed). 
    Thus, this gives a good idea of the asymptotic behavior in this regime: Error exponentially decays with the number of layers, and polynomially decays with the number of nodes in the network.

    \item[Section \ref{ned}: The case where the input and output dimensions are equal.] 
    \hfill \\
    In this regime, we give constructions for networks that can translate between single dimensional uniform and normal distributions. These constructions incur $\eps$ error in Wasserstein distance using only $\polylog(1/\eps)$ nodes.

    \item[Section \ref{ngd}: The case where the input dimension is greater than the output dimension.]
    \hfill \\
    In this regime, we show that even with trivial networks, increased input dimension can sometimes improve accuracy.

\end{description}

In the course of proving the above results, we show several lemmas of independent interest.
\begin{description}
    \item[Multivariable affine complexity lemma.]
    \hfill \\
    For a function $f:\R^{n_0} \to \R^{d}$ computed by a neural network with $N$ nodes and $L$ layers and ReLU nonlinearities, the domain of $f$ can be partitioned into $O \left(\frac{N}{n_0L} \right)^{n_0 L}$ convex (polyhedral) pieces such that $f$ is affine on each piece. This is extends prior work, which considered networks with only univariate input \citep{telgarsky2016benefits}.

    \item[Taylor series approximation.]
    \hfill \\
    Univariate functions with quickly decaying Taylor series, such as $\exp, \cos$, and the CDF of the standard normal, can be approximated on domains of length $M$ with networks of size $\poly(M, \ln(1/\eps))$.
    This idea was been explored before by \citet{yarotsky2017error}; the key difference between this work and the prior is that our results apply directly to arbitrary domains.

    \item[Function inversion through binary search.]
    \hfill \\
    The inverses of increasing functions with large enough slope can be approximated efficiently, provided that the functions themselves can be approximated efficiently. While this technique does not provide uniform bounds on the error, we show that it provides approximations that are good enough for generative networks to have low error.

\end{description}

Detailed proofs of most theorems and lemmas can be found in the appendix.

\subsection{Related Work} 

Generative networks have become popular in the form of Generative Adversarial Networks (GANs), introduced by \citet{goodfellow2014generative}; see for instance
\citep{creswell2018generative} for a survey of various GAN architectures.
GANs are trained using a discriminator network, an auxiliary neural network which tries to prove the distance from the simulated distribution to the true data distribution is large.
The generator is trained by gradient descent to minimize the distance given by the adversary network. 
Wasserstein GANs (or WGANs) are GANs which use an approximation of the Wasserstein distance as this notion of distance. The concept of Wasserstein distance comes out of the theory of optimal transport, as discussed in \citet{villani2003topics}, and its use as a performance metric is expounded in \citet{arjovsky2017wasserstein}. WGANs have shown success in various generation tasks \citep{osokin2017gans, donahue2018synthesizing, chen2017face}. While this paper uses the Wasserstein distance as a performance metric, we are not concerned with the training process, only the representational capabilities of the networks.

Many of the results in this paper build out of the results on the representational power of neural nets as function approximators.
These results first focused upon approximating continuous functions with a single hidden layer
\citep{nn_stone_weierstrass,cybenko}, but recently branched out to deeper networks
\citep{telgarsky2016benefits,ohad_nn_apx,yarotsky2017error, montufar2014number}.
A concurrent work in this area is \citet{zhang2018tropical}, which uses tropical geometry to analyze deep networks. This work produced a result on the number of affine pieces of deep networks \citep[Theorem 6.3]{zhang2018tropical}, which matches our bound in \autoref{MultidimAffinePiece}. This bound was originally suggested in \citet{montufar2014number}.
The present work relies upon some of these recent works (e.g., affine piece counting bounds,
approximation via Taylor series), but develops nontrivial extensions (e.g., multivariate
inputs and outputs with tight dimension dependence, less benign Taylor series).

The representational capabilities of generative networks have previously been studied by \citet{lee2017ability}. That paper provides a result for the representation capabilities of deep neural networks in terms of ``Barron functions'', first described in \citet{barron}, which are functions with certain constraints on their Fourier transform. \citet{lee2017ability} showed that compositions of these Barron functions could be approximated well by deep neural networks. Their main result with respect to the representation of distributions was that the result of mapping a noise distribution through a Barron function composition could be approximated in Wasserstein distance by mapping the same noise distribution through the neural network approximation to the Barron function composition. These techniques do not readily permit the analysis of target distributions which are not images of the input space under these Barron functions.

The Box-Muller transform \citep{box1958note} is a computational method for simulating bivariate normal distributions using uniform distributions on the unit (2-dimensional) square. The method is a general algorithm, but it is possible to simulate the transform with specially constructed neural nets, to prove theorems similar to those in \autoref{ned}. In fact this was our original approach; an overview of the Box-Muller implementation can be found in \autoref{ned}.

\subsection{Notation for Neural Networks}

We define a neural network with $L$ layers and $n_i$ nodes in the $i$th layer as a composition of functions of the form
\[ 
    A_L \circ \sigma_{n_{L-1}} \circ A_{L-1} \circ \sigma_{n_{L-2}} \circ \cdots 
    \circ \sigma_{n_1} \circ A_1.
\]
Here $A_i : \R^{n_{i-1}} \to \R^{n_i} $ is an affine function. That is, $A_i$ is the sum of a linear function and a constant vector. $\sigma_{k} : \R^k \to \R^k$ is the $k$-component pointwise ReLU function, where the ReLU is the map $x \mapsto \max\{x,0\}$. The total number of nodes $N$ in a network is the sum $ \sum_{i=0}^L n_i$. We will sometimes use $n = n_0$ to refer to the input dimension and $d = n_L$ to refer to the output dimension.

Since a neural network is a composition of piecewise affine functions, it is piecewise affine. The number of affine pieces of a function $f$ will be denoted $N_A(f)$ or just $N_A$.

When $\mu$ is a distribution, we will adopt the notation of \citet{villani2003topics} and use $f\# \mu$ to denote the pushforward of $\mu$ under $f$, i.e., the distribution obtained by applying $f$ to samples from $\mu$. We will use $U(A)$ to denote the uniform distribution on a set $A \subset \R^n$, and $m(A)$ to denote the Lebesgue measure of that set. We will use $\cN$ to denote a normal distribution, which will always be centered on the origin and have unit covariance matrix.

\section{Increasing the Dimensionality of Noise} \label{nld}

How easy is it to create a generator network that can output more dimensions of noise than it receives? 
It is common in practice to use a far greater output dimension.
Here, we give both upper and lower bounds showing that an increase in dimension can
require a large, complicated network.

\subsection{Constructions for the Uniform Hypercube}

For this section, we restrict ourselves to the case of input and output distributions which are uniform.
To motivate our techniques, we can simplify our problem even further: We could ask how one might approximate a uniform distribution on the unit square using the uniform distribution on the unit interval. We see that we are limited by the fact that the range of the generator net is some one-dimensional curve in $\R^2$, and so the distribution that the generator net produces will have to be supported on this curve. We will want each point of the unit square to be close to some point on the curve so that the mass of the square can be transported to the generated distribution. We are therefore led to consider some kind of (almost) space filling curve.   
An excellent candidate is the graph of the iterated tent map, shown in Figure \ref{fig:2dtentmap}. This function has been useful in the past \citep{montufar2014number, telgarsky2016benefits} since it is highly nonlinear and it can be shown that neural networks must be large to approximate it well. We can create a construction for the univariate to multivariate network which uses tent maps of varying frequencies to fill space.

\begin{figure} 
  \includegraphics[width=0.5\textwidth]{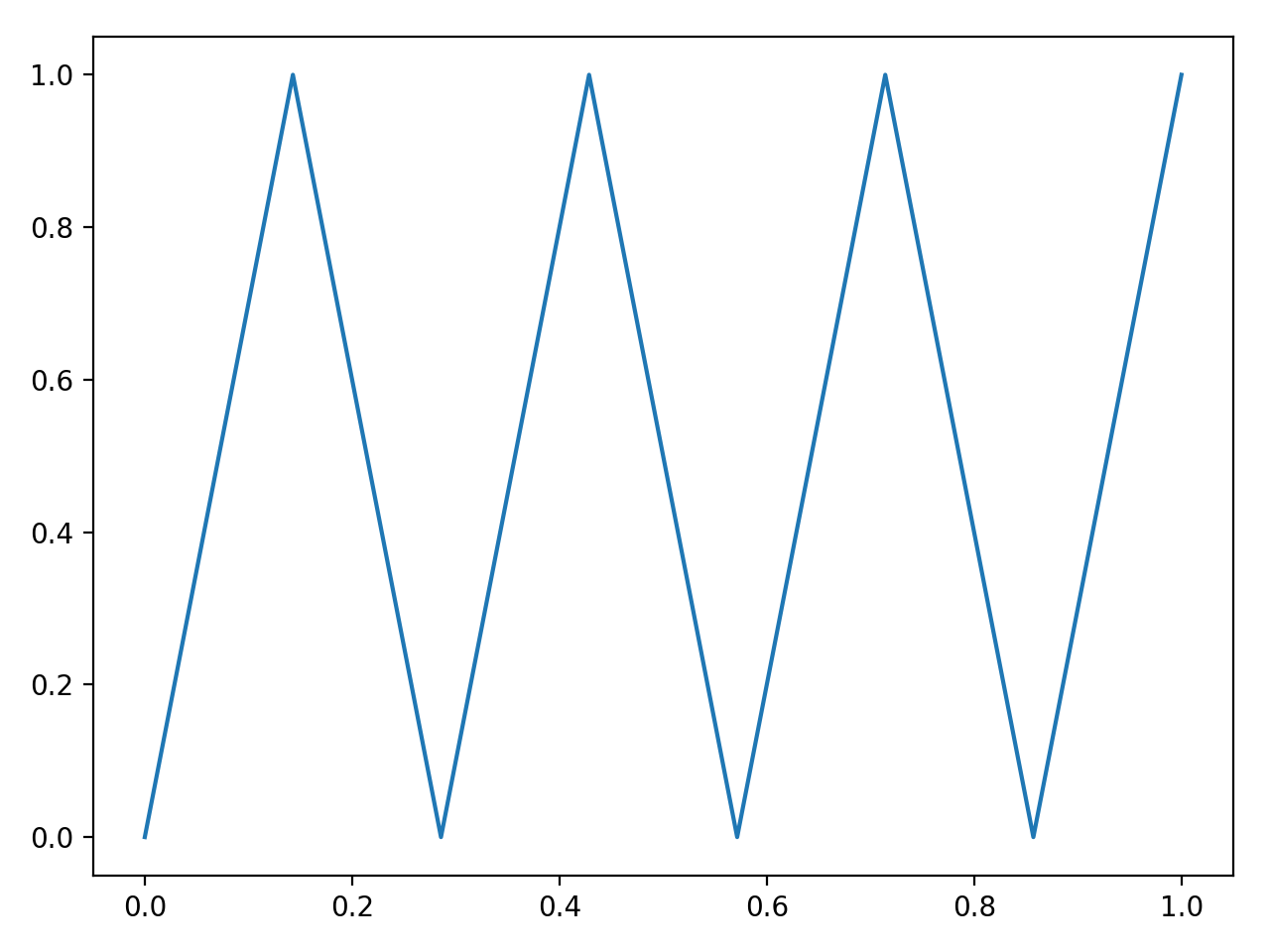}
  \includegraphics[width=0.4\textwidth]{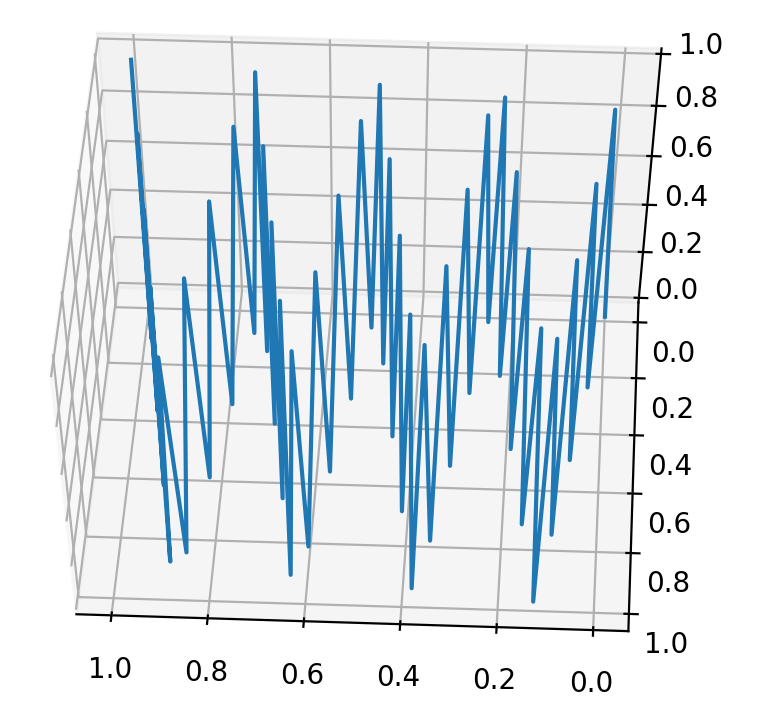}
  \caption{Examples of paths that come near every point in the unit square and the unit cube. }
\end{figure} \label{fig:2dtentmap}

The tentmap construction, which appears in \citep{montufar2014number} and is given in full in the appendix, achieves the following error:
\begin{theorem} \label{onedtentmapconstruction}
Let $\mu$ and $\nu$ respectively denote uniform distributions on $[0,1]$ and $[0,1]^d$. Given any number of nodes $N$ and number of layers $L$ satisfying $N > dL$, we can construct a generative network $f:[0,1] \to [0,1]^d$ such that
    \begin{equation}
      W(f \# \mu, \nu) \le \sqrt{d} \left\lfloor \frac{N - dL}{L} \right\rfloor^{ -\lfloor \frac{L}{d-1} \rfloor}.
      \label{eq:onedtentmap}
    \end{equation}
\end{theorem}
Thus, as the size of the network grows, the base in \autoref{eq:onedtentmap} grows proportionally to the average width of the network, and the exponent grows proportionally to the depth of the network (while being inversely proportional to the number of outputs). The $N > dL$ requirement comes from using some nodes to carry values forward --- If we were to allow connections between non-adjacent layers, this requirement would go away and $N$ would replace $N - dL$ in the theorem statement.

We now consider the case where our input noise dimension is larger than 1. In this case, we see that one possible construction involves dividing the output dimensions evenly amongst the input dimensions and then placing multiple copies of the above described construction in parallel. This produces the following bound:

\begin{theorem} \label{multidtentmapconstruction}
    Let $\mu$ and $\nu$ respectively denote uniform distributions on $[0,1]^n$ and $[0,1]^d$. Given any number of nodes $N$ and number of layers $L$ satisfying $N > dL$, we can construct a generative network $f:[0,1]^n \to [0,1]^d$ such that
    \[
        W(f \# \mu, \nu) \le
        \sqrt{d} 
        \left\lfloor 
            \frac{N - dL}{nL} 
        \right\rfloor^{-\left\lfloor \frac{L}{\lceil \frac{d-n}{n} \rceil} \right\rfloor} 
        = O\left(\frac{N}{nL} \right)^{-O(\frac{nL}{d})},
    \]
    where the big-$O$ hides factors of $d$ in the base, and constant factors in the exponent.
\end{theorem}

Note that this generalizes \autoref{onedtentmapconstruction}. The proof can be found in the appendix. This bound is at its tightest when $d$ is a multiple of $n$, in which case $\frac{d-n}{n}$ is an integer, and the exponent matches exactly that in the lower bound determined later. The construction works more smoothly with this even divisibility because the output nodes can be evenly split among the inputs, and it is easier to parallelize the construction.

\subsection{Lower Bounds for the Uniform Box }

We now provide matching lower bounds. For this, it suffices to count the affine pieces. Bounds on the number of affine pieces have been proved before, but only with univariate input
\citep{telgarsky2016benefits}; here we allow the network input to be multidimensional.
\begin{lemma} \label{MultidimAffinePiece}
    Let $f:\R^{n_0} \to \R^d$ be a function computed by a neural network with at most $N$ total nodes and $L$ layers. Then the domain of $f$ can be divided into $N_A$ convex pieces on which the function is affine, where
    \[ 
        N_A \le \left( e\frac{N}{n_0L}+e \right)^{n_0L} .
    \]
\end{lemma}

This lemma has also been proven in concurrent work \citep{zhang2018tropical} using the techniques of tropical geometry. Our proof works essentially by induction on the number of layers: We look at the set of possible activations of the $i$th layer, we see that it is a union of convex affine sets of dimension at most $n_0$. The application of the ReLU maps each of these convex affine sets into $O(n_i)^{n_0}$ convex affine sets, where $n_i$ is the number of nodes in the $i$th layer.

The one-dimensional tent map construction tells us that for a given number of nodes and number of layers, we can construct a function with a number of affine pieces bounded by the size of the network. When constructing multidimensional input networks with a high number of affine pieces, we can always parallelize several of these tentmaps to get a map with the product of the number of pieces in the individual networks. What this lemma guarantees is that, up to a constant factor in the number of nodes, this construction is optimal for producing as many affine pieces as possible. This gives us confidence that our parallelized tent map construction for low-dimensional uniform to high-dimensional uniform may be close to optimal.

To show that our construction is optimal, we need to show that it approximates the high-dimensional uniform distribution about as accurately as any piecewise affine function with the same number of pieces $N_A$. To do this, we will use the fact that the range of a piecewise affine function is a subset of the union of ranges of its constituent affine functions. We then show that any distribution on a union like this is necessarily distant from the target uniform distribution. 

\begin{theorem} \label{distributiondimensionalitygap}
    Let $B$ be a bounded measurable subset of $\R^d$ of radius $l$, let $f: \R^n \to \R^d$ be piecewise affine with $n < d$, and let $P$ be any distribution on $\R^n$. The Wasserstein distance between $f\#P$ and the uniform distribution $U_B$ on $B$ has the following lower bound:
    \[ 
        W(U_B, f\#P) \ge k \left( l^{-n} \frac{m(B)}{N_A} \right)^\frac{1}{d-n},
    \]
    where $k$ depends on $n$ and $d$.
\end{theorem}

Note that our technique can produce bounds not just for the unit cube on $n$-dimensions but for any uniform distribution on any bounded subset of $\R^n$ such as the sphere. When we combine this with our analysis of $N_A$ in \autoref{MultidimAffinePiece}, we get a lower bound result for a given number of nodes and layers.

\begin{theorem} \label{nlderrorlowerbound}
    Let $\mu$ and $\nu$ respectively denote uniform distributions on $[0,1]^n$ and $[0,1]^d$. Given any number of nodes $N$ and number of layers $L$, for any generative network $f:[0,1]^n \to [0,1]^d$, we have
    \[
        W(f \# \mu, \nu) \ge k \left( e\frac{N}{nL}+e \right)^{-\frac{nL}{d-n}} = O\left(\frac{N}{nL} \right)^{-\frac{nL}{d-n}} , 
    \]
    where the big-$O$ hides factors of $n$ and $d$ in the base.
\end{theorem}

\begin{proof}
  This follows from applying \autoref{distributiondimensionalitygap} taking $f$ as a neural network with the affine piece bound from \autoref{MultidimAffinePiece}, and $P$ as $\mu$, the uniform distribution on $[0,1]^n$.
\end{proof}

\section{Transporting between Univariate Distributions} \label{ned}

The two most common distributions used in generative networks in practice are the uniform distribution, and the normal. How easily can one of these distributions can be used to approximate the other? We will deal with the simplest case where our input and output distributions are one-dimensional. If we can construct a neural net for this case, we can parallelize multiple copies of the net if we want to move between normal and uniform distributions in higher dimensions.

\subsection{Approximation of a Uniform Distribution by a Normal}

Perhaps the simplest idea for approximating a uniform distribution with a generative network with normal noise is to let the network approximate $\Phi$, the cumulative distribution function of the normal. 
To approximate $\Phi$, we will approximate its Maclaurin series: 
\[
  \Phi(z) = 
  \frac12 
  + \frac{1}{\sqrt{2\pi}}
     \sum_{n=0}^\infty \frac{(-1)^n z^{2n+1} }{n! (2n+1) 2^n}.
\]
This series has convergence properties which allow a network based on its truncation to work.

\citet[Proposition 3c]{yarotsky2017error} showed that $f:(x,y) \mapsto xy$ over $[-M, M]^2$ can be efficiently approximated by neural networks, 
  in the sense that there is a network with $O(\ln(1/\eps) + \ln(M))$ nodes and layers computing a function $\hat{f}$ with $|f - \hat{f}| \le \eps$.  
\citet{yarotsky2017error} uses this to show that certain functions with small derivatives could be approximated well. We will show a similar result, which depends on the good behavior of the Taylor series of $\Phi$. 

Naturally, if $\tilde{f}$ is a neural network approximating $f$ and $\tilde{g}$ approximates $g$, then we can compose these approximations to get an approximation of the composition. In particular, if $g$ has a Lipschitz constant, then the composition approximation will depend on this Lipschitz constant and the accuracies of the individual approximations. A ``composition lemma'' to this effect is included in the appendix as \autoref{compositionlemma}. We will use this idea several times to construct a variety of function approximations.

We will now consider the method of approximating functions by approximating their Taylor series with neural networks. To do this, we first demonstrate a network which takes a univariate input $x$ in $[-M, M]$ and returns the multivariate output $(x^0, x^1, x^2, \ldots, x^n)$.

\begin{theorem} \label{powerseries} 
    The function $f: x \mapsto (x^0, \ldots, x^n)$ on $[-M, M]$ can be computed uniformly to within $\eps$ 
    by a neural network of size $\poly(n, \ln(M), \ln(1/\eps))$.
\end{theorem}

The proof relies on iteratively composing the multiplication function $x^i = x^{i-1} \cdot x$ using the ``composition lemma'' to get each of the $x^i$.

Now that we know the size required to approximate the powers of $x$, we may use this to approximate the Maclaurin series of $\Phi$.

\begin{theorem} \label{normalcdfapprox}
    The function $\Phi$ can be approximated uniformly to within $\eps$ by a network of size $\poly(\ln(1/\eps))$.
\end{theorem}

To show this we apply \autoref{powerseries} with a suitable choice of $M$ and $n$ and then use the monomial approximations to get a Taylor approximation of $\Phi$. Knowing that we can approximate $\Phi$ well, we can give a precise bound on the Wasserstein distance of this construction.

\begin{theorem}
  We can construct a generative network with $\polylog(1/\eps)$ nodes and univariate normal noise that can output a distribution with Wasserstein distance $\eps$ from uniform.
\end{theorem}

\begin{proof}
  Using \autoref{normalcdfapprox}, let $\tilde{\Phi}: \R \to [0,1]$ approximate $\Phi$ uniformly to within $\frac{\eps}{2}$. Consider the coupling $\pi$ between the output of this network and the true uniform distribution which consists of pairs $(\Phi(z), \tilde{\Phi}(z))$, where $z$ is normally distributed:
  \[
    W\del{U([0,1]), \tilde{\Phi}\#\cN}
    = W\del{\Phi \# \cN, \tilde{\Phi}\#\cN}
    \le
    \int_{\R} |\Phi(z) - \tilde{\Phi}(z)| \cdot \frac{1}{\sqrt{2\pi}} e^{-z^2/2} dz.
  \]
  But since $|\Phi - \tilde{\Phi}|$ is less than $\eps$ everywhere, this integral is no more than $\eps$, so we can indeed create a generative network of $\polylog(1/\eps)$ nodes for this task.
\end{proof}

\subsection{Approximation of a Normal Distribution by Uniform}

Having shown that normal distributions can approximate uniform distributions with $\polylog(1/\eps)$ nodes, let's see if the reverse is true. For this we'll need a few lemmas.

For analytic convenience, a few of our intermediate constructions will use networks with both ReLU activations, as well as step functions  $H(x) = \1[x > 0]$. 
Networks with these two allowed activations have a convenient property which allows them to be used to study vanilla ReLU networks: If there is a ReLU/Step network approximating a function $f$ uniformly, then $f$ can be uniformly approximated by a comparably-sized network on all but an arbitrarily small positive subset of its domain.

\begin{lemma} \label{stepimpliesgiveup}
  Let $\mu$ be a measure, and $A$ a measurable set with $\mu(A) < \infty$.
    Suppose $f: \R^{n} \to \R^{d}$ can be approximated uniformly to within $\eps$ on $A$
    by a function $\tilde f$ computed by a ReLU/Step network with $N$ nodes. 
    Then for any $\zeta > 0$, there exists a ReLU network with $O(N)$ nodes which approximates $f$ to within $2\eps$ on a set $B$ where $A \setminus B$ has measure less than $\zeta$.
\end{lemma}

\begin{proof}
    Note that while a ReLU neural network cannot implement the step function, it can implement the following approximation: 
    \[
        s_\delta(x)
        = 
        \begin{cases}
            0 & \text{if } x \le 0,\\
            x/\delta & \text{if } 0 \le x \le \delta ,\\
            1 & \text{if } x > \delta.
        \end{cases}
    \]
    In fact, this approximation to the step function can be implemented with a 4-node ReLU network. If we replace every step function activation node in our architecture with a copy of this four node network, we get an architecture of size $O(N)$. 
    With this architecture, we can compute each of a sequence $(f_n)$ of functions, 
        where in $f_n$, all step functions from our old network are replaced by $s_{1/n}$. 
    For any $x$ in $A$, consider the minimum positive input to the step function which occurs in the computation graph. 
    If $\delta = 1/n$ is less than this minimum, then $f_n(x) = f(x)$, so this sequence converges pointwise to $\tilde{f}$. 
    Egorov's theorem \citep[pp. 290, Theorem 12]{kolmogorov1975introductory} now tells us that $(f_n)$ converges to $\tilde{f}$ uniformly on a set $B$ that satisfies the $\mu(A \setminus B) < \zeta$ requirement. 
    Thus, there is an $f_n$ that approximates $\tilde{f}$ to within $\eps$ uniformly on this $B$, and $f_n$ therefore approximates $f$ uniformly on $B$ to within $2\eps$.
\end{proof}
 
This lemma has a useful application to generative networks: If we make $\zeta$ sufficiently small, the mass of the noise distribution on $A\setminus B$ is arbitrarily small. Therefore, we can make $\zeta$ small enough that the impact of the mistake region on the Wasserstein distance is negligible.

We now would like to approximate some function that maps the uniform distribution to the normal in this powerful format. Complementing the use of the normal CDF $\Phi$ in the previous subsection, here we will use its inverse $\Phi^{-1}$. Since we conveniently have already proved that $\Phi$ is efficiently approximable, we would like a general lemma that allows us to invert this.

\begin{lemma} \label{inverseapprox}
    Let $f:[a,b] \to [c,d]$ be a strictly increasing differentiable function 
    with $f'$ greater than a constant $L$ everywhere,
    and let $f$ be approximated to within $\eps$ by a network of size $N$.
    Then (for any $\zeta > 0$), $f^{-1}$ can be approximated to within $(b-a)2^{-t} + \eps L $ on (all but a measure $\zeta$ subset of) $[c,d]$ by a network of size $O(tN)$.
\end{lemma}

The proof of this lemma constructs a neural network that executes $t$ iterations of a binary search on $[a,b]$, using $t$ copies of the approximation to $f$ to decide which subinterval to narrow in on.
Applying this lemma to our approximation theorem for the normal CDF gives us an approximation of the inverse of the normal CDF.

\begin{theorem} \label{inversenormalcdfapprox}
    For any $\zeta> 0$, the function $\Phi^{-1}$ can be approximated to within $\eps$ by a network of size $\polylog(1/\eps)$ 
    on  $[\Phi(-\ln(1/\eps)), \Phi(\ln(1/\eps))] \setminus A$ where $A$ is of measure $\zeta$.
\end{theorem}       

\begin{proof}
    By \autoref{normalcdfapprox} we can get the normal CDF $\Phi$ to within 
    $\eps^{\ln(1/\eps) + 1}$ with $\polylog(1/\eps)$ nodes. 
    Using \autoref{inverseapprox} with $t = O(\ln(1/\eps))$, 
        if we choose $a = -\ln(1/\eps), b = \ln(1/\eps)$ 
        then the Lipschitz constant of $\Phi^{-1}$ on this interval is 
        \[
            \Phi'(\ln(1/\eps))^{-1} = O(e^{\ln(1/\eps)^2/2}) = O(\eps^{-\ln(1/\eps)}),
        \]
    and so \autoref{inverseapprox} gives a total error on the order of $\eps$.
\end{proof}

With this approximation, we can get a generative network approximation. Since the tails of the normal distribution are small, we can ignore them by collapsing the mass of the tails into a bounded interval. Then, by setting $\zeta$ sufficiently small that the Wasserstein distance contributed by the error region is negligible, our approximation can be shown to be within $\eps$ of the normal.

As a final lemma, we note the following observation
\begin{proposition} \label{cdfdifference}
    For two distributions on $\R$, their Wasserstein distance is equal to the $L^1$ integral of the difference of their CDFs.
\end{proposition}

For a proof, see \citep[remark 2.19.ii]{villani2003topics}. For an intuition, note that moving a mass $m$ from $a$ to $b$ on a one-dimensional distribution changes the CDF of the distribution on $[a,b]$ by $m$.

With these in place, we use get a bound for the uniform to normal construction.

\begin{theorem} \label{uniformtonormal}
    A generative network with $\polylog(1/\eps)$ nodes and univariate uniform noise can output a distribution with Wasserstein distance $\eps$ from a normal distribution.
\end{theorem}

The proof is an application of \autoref{inversenormalcdfapprox} and \autoref{cdfdifference}.

\subsubsection{The Box-Muller Transform}

We've established the bound we sought (approximation of a normal distribution via uniform),
but in this section we'll show that a curious classical construction also fits the bill,
albeit in two dimensions.
The \emph{Box-Muller transform} \citep{box1958note} comes from the observation that if $X_1$ and $X_2$ are two independent uniform distributions on the unit interval, then if we define 
\begin{equation} \label{boxmullertransform}
    Z_1 := \sqrt{-2 \ln(X_1)} \cos(2 \pi X_2)
    \qquad\textup{and}\qquad
    Z_2 := \sqrt{-2 \ln(X_1)} \sin(2 \pi X_2),
\end{equation}
then $Z_1, Z_2$ are independent and normally distributed. \autoref{boxmullertransform} comes from the interpretation of $\sqrt{-2\ln(X_1)}$ and $2\pi X_2$ as $r$ and $\theta$ in a polar-coordinate representation of the pair of normals. While this method is not as powerful as the CDF approximation method, in that it requires two dimensions of uniform noise in order to work, it still suggests an idea for a similar theorem to \autoref{inversenormalcdfapprox}.

\begin{theorem}
    For any $\zeta> 0$, the function $X_1, X_2 \mapsto \sqrt{-2 \ln(X_1)} \cos(2 \pi X_2), \sqrt{-2 \ln(X_1)} \sin(2 \pi X_2) $ can be approximated to within $\eps$ by a network of size $\polylog(1/\eps)$ 
    on  $[0,1]^2 \setminus A$ where $A$ is of measure $\zeta$.
\end{theorem}

We provide the following proof sketch:
\begin{itemize}
    \item The $\cos$ and $\sin$ functions (and the $\exp$ function) can be efficiently computed for much the same reason that $\Phi$ can: Their Taylor expansion coefficients decay rapidly.
    \item The $\ln$ function can be approximated in $[1/2, 3/2]$ using the Taylor series for $\ln(1 + x)$. For inputs outside this interval, we can repeatedly multiply double/halve the input until we reach $[1/2, 3/2]$, use the approximation we have, then add in a constant depending on the number of times we doubled or halved.
    \item The square root function, and in fact all functions of the form $x \mapsto x^\alpha$ for $\alpha > 0$, can be approximated using the approximations for $\exp$ and $\ln$ and the identity $x^\alpha = \exp( \alpha \ln(x))$.
    \item Putting these together, as well as the approximation for products from \citet{yarotsky2017error}, we get the result.
\end{itemize}

\section{From Many Dimensions to One} \label{ngd}

This section will complete the story by seeing what is gained in transporting many dimensions into one.

To begin, let's first reflect on the bounds we have.
So far, we have shown upper bounds on neural network sizes that are polylogarithmic in $1/\eps$.
A careful analysis of the previous subsection shows that the construction uses $O(\ln^5(1/\eps))$ for normal to uniform and $O(\ln^{18}(1/\eps))$ for the uniform to normal. We would like to know how close to optimal these exponents are. The goal of this subsection is to quickly establish that the lower bound for this exponent is at least 1. To do this, we will make some use of the affine piece analysis from Section \ref{nld}.

Note that piecewise affine functions acting on the uniform distribution have structure in their CDF, since they are a mixture of distributions induced by each individual affine piece:
\begin{proposition} \label{piecewiseaffinecdf}
    For a piecewise affine function $f:[0,1]\to \R$ with $N_A$ pieces, the CDF of the a distribution $f \# U([0,1])$ is a piecewise affine function with at most $N_A + 2$ pieces.
\end{proposition}

So if we can establish a bound on the accuracy with which a piecewise affine function can approximate the normal CDF, we can use the univariate affine pieces lemma above to lower bound the accuracy of any uniform univariate noise approximation of the normal. A helpful bound is given in \citet{DBLP:journals/corr/SafranS16}, from which we get: 
\begin{lemma} \label{affineapproxbound}
    Let ${f}$ be a univariate piecewise affine function with $N_A$ pieces. Then
    \[
        \int_a^b |\Phi(x) - {f}(x)| dx \ge \frac{K}{N_A^4}
    \]
    for some constant $K$.
\end{lemma}

Putting this together with \autoref{piecewiseaffinecdf} and \autoref{MultidimAffinePiece}, we see:

\begin{theorem}
    A generative network taking uniform noise can approximate a normal with Wasserstein accuracy exponential in the number of nodes.
\end{theorem}

Or in other words, approximation to accuracy $\eps$ requires at least $O(\log(1/\eps))$ nodes.

Clearly, if we wish to approximate a low-dimensional uniform distribution with a higher-dimensional one, all we need to do is ignore some of the inputs and spit the others back out unchewed. The same goes for normal distributions. Is there any benefit at all to additional dimensions on input noise when the target distribution is a lower dimension?

Interestingly, the answer is yes. Considering the case of approximating a univariate normal distribution with a high dimensional distribution, we note that there is the simplistic approach which involves summing the inputs and reasoning that the output is close to a normal distribution by the Berry-Esseen theorem.

\begin{theorem} \label{berryesseennonuniform}
    The distribution given by summing $n$ uniform random variables on $[0,1]$ and normalizing the result has a Wasserstein distance of $O(\frac{1}{\sqrt{n}})$ from the standard normal distribution.
\end{theorem}

Note that the above approach does not use any nonlinearity at all. It simply takes advantage of the fact that projecting a hypercube onto a line results in an approximately normal distribution. This theorem suggests another way of approaching \autoref{uniformtonormal}: Use the results of \autoref{nld} to increase a 1-dimensional uniform distribution to a $d$-dimensional uniform distribution, then apply \autoref{berryesseennonuniform} as the final layer of that construction to get an approximately normal distribution. Unfortunately, this technique does not prove the $\polylog(1/\eps)$ size: it is necessary for $\frac{1}{\sqrt{d}} \approx \eps$, which means the size of the network (indeed, even the size of the final layer of the network) is polynomial in $1/\eps$.

\section{Conclusions and Future Work}

One might ask with regards to Section \ref{ned} if there are more efficient constructions than the ones found in this section, since there is a gap between the upper and lower bounds. There are other approaches to the uniform to normal transformation, such as the Box-Muller method \citep{box1958note} we discuss. Future work could modify this or other methods to tighten the bounds found in this section.

An interesting open question is whether the results of Section \ref{ned} can be applied more generally to multidimensional distributions. Suppose for example that we have a neural network that pushes a univariate uniform distribution into a univariate normal distribution. We can take $d$ copies of this network in parallel to get a network which takes $d$-dimensional uniform noise, and outputs $d$-dimensional normal noise. Is a parallel construction of the form described here the \textit{most efficient} way to create a network that pushes forward a $d$-dimensional uniform distribution to a $d$-dimensional normal? For that matter, if $f: \R^d \to \R^d$ is of the form of a univariate function evaluated componentwise on the input, is the best neural network approximation for $f$ of a given size a parallel construction?

Another future direction is: To what extent do training methods for generative networks relate to these results? The results in this paper are only representational; they provide proof of what is possible with hand-chosen weights. One could experiment with training methods to see whether they create the ``space-filling'' property that is necessary for optimal increase of noise dimension. Alternatively, one could experiment with real-world datasets to see if changing the noise distributions while simultaneously growing or shrinking the network leaves the accuracy of the method unchanged. We ran some simple initial experiments measuring how well GANs of different architectures and noise distributions learned MNIST generation, and we found them inconclusive; in particular, we could not be certain if our empirical observations were a consequence purely of representation, or some combination of representation and training.

\subsection*{Acknowledgements}

The authors are grateful for support from the NSF under grant
IIS-1750051, and for a GPU grant from NVIDIA.

\bibliographystyle{plainnat}
\bibliography{bib}

\clearpage
\appendix

\section{Omitted proofs}

\subsection{Proof of \autoref{multidtentmapconstruction}}

Intuitively, our construction works as follows: Each output of the network will be a tentmap evaluated on one of the inputs. This will fill the output space in such a way that voxels within the unit cube of a certain size are all assigned equal mass.

Before we begin, we will define the tentmap formally, and make a few observations about it.
\begin{definition}
    The $k$-piece tentmap $t_k$ is the piecewise affine function from $[0,1]$ to $[0,1]$ defined as
    \[
        t_k(x) = 
        \begin{cases}
            kx - \lfloor kx \rfloor     & \text{for } \lfloor kx \rfloor \text{ even,} \\
            1 - kx + \lfloor kx \rfloor & \text{for } \lfloor kx \rfloor \text{ odd.} \\
        \end{cases}
    \]
\end{definition}
We first note that $t_k$ can be implemented by a $k$-node, 2-layer network as
\[
    t_k(x) = \sigma(kx) + \sum_{i=1}^{k-1} 2 (-1)^i \sigma(kx - i) .
\]
To see this, note that $\sigma$ is only nonzero when its argument is positive, and in this case it is equal to its input, so the sum can be rewritten (for $x \in [0,1]$) as
\[
    t_k(x) = kx + \sum_{i=1}^{\lfloor kx \rfloor} 2 (-1)^i (kx - i) .
\]
For $\lfloor kx \rfloor$ even, we cancel out adjacent pairs of the sum to $-1$, leaving us with 
$kx + \frac{\lfloor kx \rfloor}{2} (-2) = kx - \lfloor kx \rfloor$.
For $\lfloor kx \rfloor$ odd, we cancel out adjacent pairs leaving out the first term, leaving us with 
$kx - 2 (kx-1) - \frac{\lfloor kx \rfloor - 1}{2} (-2) = 1 - kx + \lfloor kx \rfloor$. 

We also see the identity \citep[see][Lemma 3.11]{telgarsky2016benefits} 
\[
    t_j(t_k(x)) = t_{jk}(x)
\]
since along each interval $[i/k, (i+1)/k]$, we get a copy of $t_j$ or its reflection.

The construction at the heart of \autoref{multidtentmapconstruction} uses iterated tentmaps to transfer an $n$-dimensional distribution on the unit cube onto a space-filling curve on the unit cube in $d$ dimensions. To this end, we will prove a lemma justifying that this sort of space filling curve gives a bound in Wasserstein distance.

\begin{lemma}
    Let $f:[0,1]^n \to [0,1]^d$ be the map which takes $(x_1, x_2, \ldots, x_n)$ to the point
    \[
        \left(
            t_1(x_1), t_k(x_1), t_{k^2}(x_1), \ldots, t_{k^{d_1-1}}(x_1), 
            t_1(x_2), \ldots, t_{k^{d_2-1}}(x_2),
            t_1(x_n), \ldots, t_{k^{d_n-1}}(x_n)
        \right),
    \]
    where $k, d_1, d_2, \ldots, d_n$ are positive integers with $\sum_{i=1}^n d_i = d$. 
    Then 
    \[
      W\del[2]{f \# U([0,1]^n), U([0,1]^d)} \le \frac{\sqrt{d}}{k}.
    \]
\end{lemma}

\begin{proof}
    Call a subinterval in $[0,1]$ a \textit{$k^i$-interval} if it is of the form
    \[
        I_a^i = (a k^{-i}, (a+1) k^{-i})
    \]
    Where $a$ is an integer.
    We claim that the function $f$ maps each input box of the form
    \[
        I_{a_1}^{d_1} \times \cdots \times I_{a_n}^{d_n} 
    \]
    to a subset of a distinct output box of the form
    \[
        I_{b_1}^{1} \times \cdots \times I_{b_n}^{1} .
    \]
    To see this, we must show now that any two points in the same input box in $[0,1]^n$ map to the same output box in $[0,1]^d$, but that points from different boxes in $[0,1]^n$ map to different boxes in $[0,1]^d$.  
    For an input box parameterized by $(a_i)_{i=1}^n$, we note that $t_{k^{e}}$ maps a $k^{d_i}$-interval $I_{a_i}^{d_i}$ to a $k^{d_i-e}$-interval for all $0 \le e < d_i$ and to the interval $[0,1]$ for $e = d_i$. 
    Thus, for two points $x,y \in [0,1]^n$, if $x_i, y_i$ fall in the same interval $I_{a_i}^{d_i}$, then for each coordinate $j$ corresponding to coordinate $i$, the $j$th coordinate of $f(x)$ and $f(y)$ will fall in the same interval $I_{b_j}^1$. 
    If $x_i$ and $y_i$ fall in different intervals, then there will be a $j$ corresponding to $i$ such that the $j$th component of $f(x)$ and $f(y)$ fall in different $k$-intervals. Specifically, if $I_{a_i}^{e}$ contains both $x_i$ and $y_i$ but no $k^{e + 1}$ interval contains both, then $f(x)_j$ and $f(y)_j$ will fall in different $k$-intervals where $j$ is the $e$th component of the output corresponding to $i$.

    Since there are $k^d$ boxes of both the input and the output space, we see that there must be a 1-to-1 mapping from the boxes in $[0,1]^n$, to the boxes in $[0,1]^d$ containing their images under $f$. Since both types of boxes have measure $k^{-d}$ in the uniform measures on their respective unit cubes, there exists a coupling $\pi \in \Pi(f \# U([0,1]^n), U([0,1]^d))$ such that $\pi$ is supported on pairs $(x,y)$ belonging to the same box of the latter type. Then, since $|x-y| \le \frac{\sqrt{d}}{k}$ for any two points in a cube of side length $k^{-1}$, it suffices to choose any $\pi$ which arbitrarily associates points in the same cube, and the desired bound on Wasserstein distance follows.

\end{proof}

Now that we see how the tentmap-based function $f$ can achieve a low Wasserstein distance, the proof of \Cref{multidtentmapconstruction} follows by writing $f$ as a ReLU network.

\begin{proof}[Proof of \Cref{multidtentmapconstruction}]

  Our network is designed as follows: Each layer of the network has two types of nodes:
      $d$ are ``carry-forward nodes'' which copy forward the value of a specific output once it is generated by a layer,
      and the remainder are ``space-filling nodes'', which compute high-frequency tent maps. 
  We first set aside the $d$ nodes in each layer (a total of $dL$ nodes) to use as carry-forward nodes and use $x_{l,i,\texttt{carry}}$ to denote the carry-forward node in layer $l$ for output $i$. 
    The remaining $N-dL$ space-filling nodes each correspond to a specific input component. We will call $n_{l,i}$ the number of space-filling nodes corresponding to input $i$ in layer $l$. We use $x_{l,i,j,\texttt{space}}$ to denote the input of the $j$th node corresponding to input $i$ in layer $l$. 
    
    We define the weights of the nodes such that the pre-ReLU activation of the $j$th node corresponding to input $i$ in layer $l$ is
    \[
        x_{l,i,1,\texttt{space}} := n_{l,i} t_{k_{l,i}}(x_i) - (j-1),
    \]
    where $k_{l,i} = \prod_{m=1}^{l-1} n_{m,i}$.
    
    We see by induction on $l$ that it is possible to have the activations thus: For $l=1$, we have $k_{l,i} = 1$, so the tent map we are replicating is the identity, and all the activations are affine functions of the input:
    \[
        x_{1,i,j,\texttt{space}} = n_{l,i} (x_i) - (j-1).
    \] 
    For $l > 1$, we have (using the inductive hypothesis and the sum form of the tentmap) the tentmap $t_{n_{l,i}}$ as an affine combination of the post-ReLU activations of the previous layer evaluated on $t_{k_{l,i}}(x_i)$. Thus, using the product rule, we can obtain the value
    \[
        t_{n_{l,i}}(t_{k_{l,i}}(x_i)) = t_{k_{l+1,i}}(x_i)
    \]
    as an affine combination of layer $l$. And since this value can be computed in layer $l$, so can the affine transformations 
    \[
        n_{l+1,i} t_{k_{l+1,i}}(x_i) - (j-1) = x_{l,i,j,\texttt{space}}
    \]
    for any value of $n_{l+1, i}$ and $j$.

    Now that we have established how the space-filling nodes implement the tentmap function, we specify how these tentmaps feed in to the output. For input vector $(x_1, x_2, \ldots, x_n)$ the output vector will be of the form
    \[
        \left(
            t_1(x_1), t_k(x_1), t_{k^2}(x_1), \ldots, t_{k^{\lceil d/n \rceil - 1}}(x_1), t_1(x_2), \ldots, t_1(x_n), \ldots, t_{k^{\lfloor d/n \rfloor - 1}}(x_n)
        \right),
    \]
    where $k$ is an whole number depending on $N,L,n,d$.
    The $d$ outputs are split evenly among the $n$ inputs, with each output manifesting as a tentmap evaluated on its designated input. Each input has at most $\lceil \frac{d}{n} \rceil$ and at least $\lfloor \frac{d}{n} \rfloor$ output nodes taking the form of a tentmap of that input. Note that we allow the trivial tentmap $t_1$, which is just the identity on $[0,1]$, and our construction has each input $x_i$ with exactly one output of the form $t_1(x_i)$. Our goal is to make $k$ as large as possible with the limited number of carry-forward nodes and layers available, and then to prove that the Wasserstein accuracy of this construction decreases quickly with $k$.

    We split the $N - dL$ space-filling nodes among the inputs in proportion with the number of outputs that input is responsible for. Thus, each output should get at least $\lfloor \frac{N - dL}{d} \rfloor $ nodes. Furthermore, each output that computes a nontrivial tentmap of its input is associated with a run of $\lfloor \frac{L}{\lceil \frac{d-n}{n} \rceil } \rfloor $ layers over which to distribute these nodes, so that the nodes for a certain input don't go over the total number of layers. With this in place, using the construction described above, we can guarantee a $k$ value of 
    \[
        k = 
        \left\lfloor
            \frac{\lfloor \frac{N - dL}{d} \rfloor}{\lfloor \frac{L}{\lceil \frac{d-n}{n} \rceil } \rfloor} 
        \right\rfloor^{\left\lfloor \frac{L}{\lceil \frac{d-n}{n} \rceil} \right\rfloor},
    \]
    or to lower bound this with a more manageable expression,
    \[
        k \ge 
        \left\lfloor
            \frac{\lfloor \frac{N - dL}{d} \rfloor}{ \frac{L}{\lceil \frac{d-n}{n} \rceil } } 
        \right\rfloor^{\lfloor \frac{nL}{d} \rfloor}
        \ge
        \left\lfloor
            \left( \frac{d-n}{n} \right) \frac{N - dL + d}{dL}  
        \right\rfloor^{\lfloor \frac{nL}{d} \rfloor}        
        .
    \]
    Thus, we have the Wasserstein distance upper bound
    \[
        W(f \# U([0,1]^n), U([0,1]^d)) 
        \le \sqrt{d}         
        \left\lfloor
            \left( \frac{d-n}{n} \right) \frac{N - dL + d}{dL}  
        \right\rfloor^{-\lfloor \frac{nL}{d} \rfloor}.
    \]    

\end{proof}

\subsection{Proof of \autoref{MultidimAffinePiece}}

A concurrent proof of this theorem appears in \citet{zhang2018tropical}, based on tropical geometry. Our proof is based on a lemma about how many different orthants an $n_0$-dimensional hyperplane can intersect, which turns out to be exponential in $n_0$. We then inductively track how many affine pieces exist in each layer of the network.

\begin{proof}
    
    The proof requires a lemma:
    \begin{lemma}
        A $k$-dimensional hyperplane $P$ in $\R^n$ intersects at most
        $ \displaystyle \sum_{j=0}^{k} \binom{n}{j} $
        orthants.
    \end{lemma}

    This lemma follows from \citep[Lemma 3.3]{bartlett2009neural}. We assume without loss of generality that our plane is not parallel to any unit vector in $\R^n$. We then consider the $k+1$ dimensional space containing $P$ and the origin as a copy of $\R^k$, and we project the $n$ unit vectors of the ambient $\R^n$ into this copy of $\R^{k+1}$. Applying \citep[Lemma 3.3]{bartlett2009neural}, the perpendicular spaces of the vector projections split the copy of $\R^{k+1}$ into $ 2 \sum_{j=0}^{k} \binom{n}{j} $ connected components. These perpendicular spaces correspond to the separating planes of the orthants in $\R^n$, and since $P$ touches half of the connected components, it intersects $ \sum_{j=0}^{k} \binom{n}{j} $ orthants in total.
  
    With this lemma, we can now prove the theorem. We proceed by induction on $L$. In the $L=0$ case, there are no nonlinearities in this network, and so $f$ is affine on its entire input.

    In the inductive case, we consider $f$ computed by an $L+1$ layer ReLU network. Let $g: \R^{n_0} \to \R^{n_{L}}$ represent the function computed by the first $L-1$ hidden layers of $f$, outputting the last hidden layer of $f$. 
    That is, $g = A_{L} \circ \sigma_{n_{L-1}} \circ \dots \circ \sigma_{n_2} \circ A_2 \circ \sigma_{n_1} \circ A_1$.
    By the inductive hypothesis, $\R^{n_0}$ can be partitioned into
    \[ 
        N_A(g) \le  \prod_{i=1}^{L-1} \left( \sum_{j=0}^{n_0} \binom{n_i}{j} \right) 
    \]
    convex parts $S_1, \cdots S_{N_A(g)}$ such that $g$ is affine on each. For any of these convex regions $S_k$, the image of $g(S_k)$ is a convex set in $\R^{n_L}$. 
    Consider the partition of $\R^{n_0}$ obtained by dividing each $S_k$ into subpieces according to the orthant of a points image under $g$. 
    Because each orthant is convex, the preimage of each orthant under the restriction of $g$ to $S_k$ (which is affine) is also convex. 
    Moreover, the function $f = A_{L+1} \circ \sigma_{n_L} \circ g$ is affine on each of these subpieces, because $g$ is affine on the subpieces and $\sigma_{n_L}$ is affine on each orthant (and $A_{L+1}$ is affine).
    Finally, since $g(S_k)$ is an affine image of a subset of $\R^{n_0}$, it lies in a $n_0$-dimensional hyperplane in $\R^{n_L}$, which can intersect at most $\sum_{j=0}^{n_0} \binom{n_L}{j}$ orthants. Thus, the subdivision step divides each $S_k$ into at most $\sum_{j=0}^{n_0} \binom{n_L}{j}$ subpieces. We therefore get
    \[
        N_A(f) \le N_A(g) \sum_{j=0}^{n_0} \binom{n_L}{j} \le \prod_{i=1}^L \left( \sum_{j=0}^{n_0} \binom{n_i}{j} \right) .
    \]
    We also note the upper bound on the sum
    \[
        \sum_{j=0}^{n_0} \binom{n_i}{j} 
        \le \binom{n_i + n_0}{n_0}
        \le \frac{(n_i + n_0)^{n_0}}{n_0!}
        \le \left( e \frac{n_i + n_0}{n_0} \right)^{n_0}.
    \]
    If we substitute this in above, we get the bound
    \[
        N_A(f) \le  
        \prod_{i=1}^L \left( e \frac{n_i + n_0}{n_0} \right)^{n_0}
    \]
    and since (keeping the total number of nodes fixed) this product is maximized when all layers have the same number of nodes, we get
    \[
        N_A(f) \le  
        \prod_{i=1}^L \left( e \frac{N}{n_0L} + e \right)^{n_0}
        = \left( e \frac{N}{n_0L} + e \right)^{n_0L} .
    \]
    This proves the claim.

\end{proof}

\subsection{Proof of \autoref{distributiondimensionalitygap}}

The proof of this theorem comes in a few parts:
\begin{itemize}
    \item We introduce a new notation to capture the idea of a Wasserstein distance between a distribution and a set.
    \item We prove a lemma about how distances between certain well-behaved distributions and hyperplanes can be lower bounded.
    \item We break the range of $f$ and the target distribution up into a collection of hyperplanes and distributions that can be handled by the lemma.
\end{itemize}

First, let us specify our idea of a Wasserstein distance between a distribution and a set:

\begin{definition}
    For a distribution $\mu$ and a closed, convex set $S$ on $\R^d$, we define the Wasserstein distance of the distribution from the set as
    \[ 
        W(\mu, S) := \inf_{\pi \in \Pi(\mu, S)} \int |x-y| d\pi(x, y) = \inf_{\nu \in \Pi_S} W(\mu, \nu)
    \]
    where $\Pi(\mu, S)$ is the set of joint distributions having $\mu$ as left marginal and right marginal supported on $S$, and where $\Pi_S$ is the set of all distributions supported on $S$.
\end{definition}

We can immediately note that an alternative way to view this definition is
\[
    W(\mu, S) = \int d(x, S) d\mu(x),
\]
where $d(x, S) := \inf_{y\in S} |x-y|$ represents the distance of $x$ from $S$. To see this, we claim that 
there is an optimal coupling $\pi^*\in\Pi(\mu, S)$ which attains the minimum $\int d(x,y)d\pi^*(x,y) = W(U_B,S)$ and which is supported on $(x,y)$ pairs where $y$ is the unique (since $S$ is closed and convex) closest point in $S$ to $x$. 
  To see why such a coupling is optimal, note that 
  for any other distribution $\pi \in \Pi(\mu, S)$,
    \[
      \int |x-y| d\pi(x,y)
      \geq \int \inf_{y\in S} |x-y| d\pi(x,y)
      = \int |x-y| d\pi^*(x,y).
    \]

With this new definition, we move on to a lemma which lower bounds the Wasserstein distance between a uniform distribution on an arbitrary bounded measurable set and a hyperplane:

\begin{lemma} \label{cylinderlemma}
    Let $B \subseteq B_0 \subseteq \R^d$ where $B_0$ is a ball of radius $l$, and $B$ has measure $m(B)$. Let $S$ be an $n$-dimensional hyperplane in $\R^d$. Then the Wasserstein distance between the uniform distribution on $B$ and the plane $S$ has the following lower bound:
    \[
        W(U_B, S) \ge \frac{d-n}{d-n+1} \cdot
        \left(\frac{\Gamma(\frac{d-n}{2} + 1)\Gamma(\frac{n}{2} + 1)}{\pi^{\frac{d}{2}}} l^{-n} m(B) \right)^{1/(d-n)}
    \]
\end{lemma}

\begin{proof}

  Using our new Wasserstein distance definition, we see that
    \begin{equation} \label{planeballequation}
      W(U_B, S) = \frac{1}{m(B)} \int_B d(x,S) dx,        
    \end{equation}
    where $m(B)$ represents the Lebesgue measure of $B$.
     
    We will now lower bound the integral over $B$ on the right hand side. We know $B$ is contained in $B_0$ of radius $l$ (centered at $x_0$, say), and the orthogonal projection $\Proj_S(B)$ of $B$ onto the plane $S$ is, therefore contained in $R_1 := \Proj_S(B_0)$, which is a ball of radius $l$ on the space $S$ (centered at $\Proj_S(x_0)$).
    In the orthogonal complement space to $S$, define $R_2 \subset S^\perp$ as the ball which is centered on $S$ and has radius $r$ such that $m(R_1) \cdot m(R_2) = m(B)$ 
    and define $R^*$ to be the Cartesian product of these balls in $d$-dimensional space, so that
    \begin{equation} 
        m(R^*) = m(R_1 \times R_2) = m(R_1) \cdot m(R_2) = m(B).        
    \end{equation}
    We will see that this $R^*$ can replace $B$ in \autoref{planeballequation} to provide the desired lower bound for the expression. 
    From $m(B) = m(R^*)$, we get
    \[
      m(R^*\setminus B)
      = m(R^* \cup B) - m(B)
      = m(R^* \cup B) - m(R^*)
      = m(B\setminus R^*).
    \] 
    Since $\Proj_S(B) \subseteq R_1$, we have $B \subseteq R_1 \times S^\perp$, and since $R^*$ consists of all points in $R_1 \times S^\perp$ with distance $\le r$ from $S$, for any $x \in B \setminus R^*$, we have $d(x, S) \ge r$. 
    On the other hand, $d(x, S) \le r$ for $x \in R^* \setminus B$ (since all elements of $R^*$ are within $r$ of $S$).
    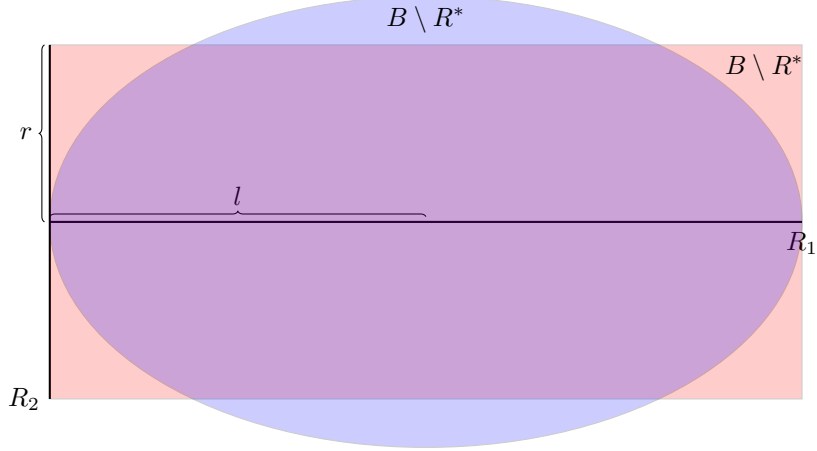
\begin{figure} \label{fig:rstarbdiagram}
    \centering
    \begin{tikzpicture}[xscale=5, yscale=3]
        \draw[thick] (-1,0) -- (1, 0) node[below] {$R_1$};
        \draw[decorate,decoration={brace,raise=2pt,amplitude=2pt}] (-1,0)  -- node[above = 3pt]{$l$} (0,0);
        \draw[thick] (-1, 3.14/4) -- (-1, -3.14/4)  node[left] {$R_2$};
        \draw[decorate,decoration={brace,raise=2pt,amplitude=2pt}] (-1,0)  -- node[left = 3pt]{$r$} (-1,3.14/4);
        \draw[fill = red, opacity = 0.2] (-1, 3.14/4) -- (-1, -3.14/4) -- (1, -3.14/4) -- (1, 3.14/4) -- cycle ;
        \draw[fill = blue, opacity = 0.2] (0, 0) circle [radius=1] ;
        \node[below] at (0, 1) {$B \setminus R^*$} ;
        \node[left,below] at (0.9, 3.14/4) {$B \setminus R^*$} ;
    \end{tikzpicture}
    \caption{Diagram of example $B$ (the ball in blue) and $R^*$ (the rectangle in red). Here $n=1$ and $d=2$.}
    \end{figure}
    Thus, we get a lower bound on the integral from \autoref{planeballequation} 
    \begin{align*}
        \int_B d(x,S) dx
        &=   \int_{B \cap R^*} d(x,S) dx + \int_{B \setminus R^*} d(x,S) dx \\
        &\ge \int_{B \cap R^*} d(x,S) dx + r \cdot m(B \setminus R^*) \\
        &=   \int_{B \cap R^*} d(x,S) dx + r \cdot m(R^* \setminus B) \\
        &\ge \int_{B \cap R^*} d(x,S) dx + \int_{R^* \setminus B} d(x,S) dx \\
        &=   \int_{R^*} d(x,S) dx,
    \end{align*}
    and so
    \begin{equation} \label{rreduction}
        W(U_B, S) \ge \frac{1}{m(B)} \int_{R^*} d(x,S) dx.
    \end{equation}

    (Note that we haven't given up much at this point; if our only restriction on $B$ was that its orthogonal projection was bounded by $l$, then we could have $R^* = B$ and the above inequality would be tight. As it is, if $l$ is much greater than $r$, then $B$ may share much overlap with $R^*$ anyway.)

    We can decompose the above integral over $R^*$ in \autoref{rreduction} into the components parallel and perpendicular to $S$:
    \begin{equation} \label{breakdown}
        \int_{R^*} d(x,S) dx = m(R_1) \cdot \int_{R_2} d(x, S) dx.
    \end{equation}
    The integral $\int_{R_2} d(x, S) dx$ is equivalent to the rotationally symmetric integral around the origin
    \begin{equation} \label{translate}
        \int_{R_2} d(x, S) dx = \int_{B_r} |x| dx,
    \end{equation} 
    where $B_r$ is the origin-centered ball of radius $r$ in $(d-n)$ dimensions. Intuitively, in a high dimensional space, most of the volume of a ball lies near its edge, so we can expect this integral to come out to about $r$ times the volume of the ball.
    We can evaluate the integral precisely by subtracting out $\int_{B_r} r - |x| dx$ and using the general formula that the volume of a cone (with $(d-n)$-dimensional base) is $\frac{1}{d-n+1}$ that of the cylinder with the same base and height: 
    \begin{equation} \label{conetrick}
        \int_{B_r} |x| dx
        = \int_{B_r} r dx - \int_{B_r} \del{r - |x|} dx
        = r \cdot m(B_r) - r \cdot \frac{1}{d-n+1} m(B_r) 
        = r \cdot m(B_r) \cdot \frac{d-n}{d-n+1}.
    \end{equation}
    Putting together \Crefrange{breakdown}{conetrick}, we get the integral from \autoref{rreduction} in terms of $r$, $m(B)$, $d$, and $n$:
    \[
        \int_{R^*} d(x,S) dx
        = r \cdot m(R_1) \cdot m(R_2) \frac{d-n}{d-n+1}
        =  r \cdot m(B) \cdot \frac{d-n}{d-n+1}.
    \]

    All we need to complete the bound is to compute $r$. 
    Recall we have $R_1$ as an $n$-dimensional ball of radius $l$ parallel to $S$,
    and $R_2$ is a $(d-n)$-dimensional ball of radius $r$ orthogonal to $S$ and centered on $S$. We use the fact that $m(R_1) \cdot m(R_2) = m(B)$ to get
    \[
        \left( \frac{\pi^{\frac{d-n}{2}}}{\Gamma(\frac{d-n}{2} + 1)} {r}^{d-n}  \right)
        \left( \frac{\pi^{\frac{n}{2}}}{\Gamma(\frac{n}{2} + 1)} l^{n}  \right)
        = m(B) ,
    \]
    which after solving for $r$ gives
    \[ 
        r =\left(\frac{\Gamma(\frac{d-n}{2} + 1)\Gamma(\frac{n}{2} + 1)}{\pi^{\frac{d}{2}}} l^{-n} m(B) \right)^{1/(d-n)}.
    \]
    Substituting this value for $r$ in \autoref{rreduction} gives
    \begin{align*} 
        W(U_B, S)
        &\ge \frac{1}{m(B)} \int_{R^*} d(x,S) dx \\
        &\ge \frac{d-n}{d-n+1} \cdot
        \left(\frac{\Gamma(\frac{d-n}{2} + 1)\Gamma(\frac{n}{2} + 1)}{\pi^{\frac{d}{2}}} l^{-n} m(B) \right)^{1/(d-n)}, \\
    \end{align*}
    as desired.
\end{proof}

Now that we have this lemma regarding distance to hyperplanes, we prove the theorem by applying the lemma to the planes on which the range of our piecewise affine function lies.

\begin{proof}[Proof of \Cref{distributiondimensionalitygap}]

    We first note that since $P$ can be any distribution on $\R^n$, $f\#P$ can be any distribution on the range of $f$. Therefore, it suffices to lower bound the Wasserstein distance between the distribution $U_B$ and the range of $f$.

    Since $f$ is piecewise affine, its range is a subset of the union of $N_A(f)$ $n$-dimensional hyperplanes in $\R^d$, which we name $S_1, \cdots, S_{N_A(f)}$. We call the union of these planes $S$, and we note that
    \[
        W(U_B, \Range f) \ge W(U_B, S),
    \]
    since any distribution supported on $\Range f$ is supported on its superset $S$. 
    
    The Wasserstein distance of $U_B$ to this set $S$ is lower bounded by the integral of $d(x,S)$ over $U_B$, since this lower bounds the integral for any $\pi \in \Pi(U_B, S)$. In fact, this is an equality, since the following correspondence gives us a $\pi$ that achieves this minimum: For each $i$, let $B_i \subseteq B$ consist of all $y$ that are nearer to $S_i$ than any other $S_j$ choosing the smaller index in case of ties. That is,
    \[
        B_i = \{ x \in B: i = \argmin_{j} d(x, S_j) \} .
    \] 
    This makes each $B_i$ a measurable set such that for $x \in B_i$, we have
    \[
        \inf_{y \in S} |x - y| = d(x, S_i),
    \]
    and by choosing $\pi \in \Pi(U_B, S)$ supported on $(x,y)$ pairs where $y$ is the closest point in $S$ to $x$, (choosing the minimum index when ambiguity arises), we get 
    \begin{align*}
        W(U_B, S) 
        &= \inf_{\pi \in \Pi(U_B, S)} \int|x - y| d\pi(x,y) \\
        &= \int d(x, S) dU_B(x)        \\
        &= \frac{1}{m(B)} \sum_i \int_{B_i} d(x, S_i) dy . 
    \end{align*}

    As we noted, these integrals can be expressed in terms of the Wasserstein distances of the uniform distributions on the $B_i$ to their respective $S_i$:
    \[ 
        W(U_B, S) 
        = \frac{1}{m(B)} \sum_i m(B_i) \cdot W(U_{B_i}, S_i) .
    \]
    We apply the lemma to lower bound this integral for each $i$, whereby
    \begin{align*}
        W(U_B, S)
        &\ge \frac{1}{m(B)} \sum_i m(B_i) \cdot \frac{d-n}{d-n+1} \cdot
        \left(\frac{\Gamma(\frac{d-n}{2} + 1)\Gamma(\frac{n}{2} + 1)}{\pi^{\frac{d}{2}}} l^{-n} m(B_i) \right)^{1/(d-n)} .\\
        \intertext{Since the summand is convex in $m(B_i)$, Jensen's inequality inequality allows replacing
        $m(B_i)$ with $m(B)/N_A$, thus}
        &\ge \frac{1}{m(B)} \sum_i \frac{m(B)}{N_A} \cdot \frac{d-n}{d-n+1} \cdot
        \left(\frac{\Gamma(\frac{d-n}{2} + 1)\Gamma(\frac{n}{2} + 1)}{\pi^{\frac{d}{2}}} l^{-n} \frac{m(B)}{N_A} \right)^{1/(d-n)} \\
        &= N_A \frac{1}{N_A} \cdot \frac{d-n}{d-n+1} \cdot
        \left(\frac{\Gamma(\frac{d-n}{2} + 1)\Gamma(\frac{n}{2} + 1)}{\pi^{\frac{d}{2}}} l^{-n} \frac{m(B)}{N_A} \right)^{1/(d-n)} \\
        &= \frac{d-n}{d-n+1} \cdot
        \left(\frac{\Gamma(\frac{d-n}{2} + 1)\Gamma(\frac{n}{2} + 1)}{\pi^{\frac{d}{2}}} l^{-n} \frac{m(B)}{N_A} \right)^{1/(d-n)} ,
    \end{align*}
    which is of the desired form.
\end{proof}

\subsection{Proof of \autoref{powerseries}}

\begin{proof} 
    We will approximate $x^k$ inductively by multiplying the approximation for $x^{k-1}$ with $x$. 
    We will ensure that we approximate each $x^k$ to within $\frac{\eps}{(2M)^{n-k}}$ (assuming $M$ is at least 1).
    If we approximate the multiplication by $x$ function to within $\frac{\eps}{2(2M)^{n-k}}$, and consider that multiplication by $x$ is $M$-Lipschitz, then using Lemma \ref{compositionlemma}, we have that if $x^{k-1}$ is approximated to within $\frac{\eps}{(2M)^{n-k+1}}$ then $x^k$ will be approximated to within 
    \[
        M \frac{\eps}{(2M)^{n-k+1}} + \frac{\eps}{2(2M)^{n-k}} = \frac{\eps}{(2M)^{n-k}} .
    \]
    By induction, this construction will indeed approximate all $x^k$ to within our specified accuracy. Analyzing the size of this network, we see that the network module computing $x^{k-1} \cdot x = x^k$ will require  
    \[
        O(\ln(2(2M)^{n-k}/\eps) + \ln(M^{k})) = O( (n-k) \ln(2M) + (n-k) + \ln(1/\eps) + k \ln(M) )
    \]
    nodes. Summing this over $k=1$ to $n-1$ produces $\poly(n, \ln(M), \ln(1/\eps))$.
\end{proof}

\subsection{Proof of \autoref{normalcdfapprox}}
\begin{proof}
    As mentioned before, $\Phi$ has the series representation
    \[
      \Phi(z) = 
      \frac12 
      + \frac{1}{\sqrt{2\pi}}
        \sum_{k=0}^\infty \frac{(-1)^k z^{2k+1} }{k! (2k+1) 2^k}.
    \]
    We consider the truncated sum
    \[
      \Phi_n(z) = 
      \frac12 
      + \frac{1}{\sqrt{2\pi}}
        \sum_{k=0}^n \frac{(-1)^k z^{2k+1} }{k! (2k+1) 2^k} ,
    \]
    where we set $n = \max\{ 2 e M^2 - 1, \log_2(2/\eps) \}$. 
    This guarantees that for $x \in [-M, M]$, the error incurred by truncating the sum is
    \begin{align*}
        |\Phi(z) - \Phi_n(z) |
        &=  
        \frac{1}{\sqrt{2\pi}} 
        \left| 
          \sum_{k=n+1}^\infty \frac{(-1)^k z^{2k+1} }{k! (2k+1) 2^k}
        \right| \\
        &\le
          \sum_{k=n+1}^\infty \frac{ |M|^{2k+1} }{k!} .
          &{\because\textup{ Stirling's inequality}}          \\
        &\le \sum_{k=n+1}^\infty \frac{M^{2k+1}}{\sqrt{2\pi} k^{k+\frac12} e^{-k}} \\
        &= \sum_{k=n+1}^\infty \frac{1}{\sqrt{2\pi e}} \left( \frac{e M^2}{k} \right)^{k + \frac12} .\\
        \intertext{Since we chose $n+1 \ge 2 e M^2$, we have}
        &\le \sum_{k=n+1}^\infty \left( \frac{1}{2} \right)^k \\
        &= 2^{-n} ,\\
        \intertext{ and since we chose $n \ge \log_2(2/\eps)$,}
        &\le \frac{\eps}{2} .
    \end{align*}
    So the total error we get by omitting these terms is no more than $\frac{\eps}{2}$. 
    We approximate each $z^{2k+1}$ to within $\eps/(2n+2)$, and then multiply each $x^i$ by its Maclaurin coefficient. 
    Since each coefficient is no more than 1 in absolute value, the errors in each of the Maclaurin terms is no more than $\eps/(2n+2)$. 
    We can therefore add all $n+1$ of these Maclaurin terms and get an error less than $\frac{\eps}{2}$ from the truncated series $\tilde{f}$, and a total error no more than $\eps$ from the function $\Phi$ in the interval $[-M, M]$. 
    Applying \autoref{powerseries}, approximating the $x^i$ to this accuracy requires
    \[
        \poly(2n+1, \ln(M), \ln(1/\eps)) = \poly(M, \ln(1/\eps))
    \]
    nodes. Take $M$ sufficiently large that $ 1 - \Phi(M/2) < \eps$, which can be done with $M = O(\log(1/\eps))$. Then, add a component to the neural network that interpolates between this approximation on $[-M, M]$ and $1$ for $z > M/2$ and $0$ for $x < -M/2$. This guarantees the network is accurate for all values of $z$.
    Note that we can guarantee the range of this approximation falls in $[0,1]$, by adding a gadget that clamps the output to this interval.
\end{proof}

\subsection{Proof of \autoref{inverseapprox}}

\begin{proof}
    We construct a ReLU/Step network which contains $t$ copies of the neural network approximating $f$, 
    as well $3t + 3$ nodes called $x_{i, \texttt{low}}, x_{i, \texttt{mid}}, x_{i, \texttt{high}}$ for $i$ in $\{0, \ldots, t\}$.
    The network assigns the initial values 
    \[
        x_{0, \texttt{low}} := a, 
        \qquad
        x_{0, \texttt{high}} := a, 
    \]
    and for all values of $i$, we compute
    \[
        x_{i, \texttt{mid}} := 
        \frac{x_{i, \texttt{low}} + x_{i, \texttt{low}}}{2} .
    \]
    Let $y \in [c,d]$ be the input to our network for computing $f^{-1}$.
    For $0 \le i < t$, we let $x_{i, \texttt{mid}}$ be the input to the $i$th copy of the network computing $f$, and call the output node of this copy $y_{i, \texttt{mid}}$, and we correspondingly call.
    If $y_i \ge y$ (which we test using a step function activation), set
    \[
        x_{i+1, \texttt{low}} := x_{i, \texttt{low}}, \qquad
        x_{i+1, \texttt{high}} := x_{i, \texttt{mid}}, 
    \]
    and otherwise, set
    \[
        x_{i+1, \texttt{low}} := x_{i, \texttt{mid}}, \qquad
        x_{i+1, \texttt{high}} := x_{i, \texttt{high}}.
    \]
    We set the network output to be $x_{t, \texttt{mid}}$.

    By induction on the construction of these values, if $\tilde{f}$ is the approximation of $f$ given by the provided network, then the interval
    $[\tilde{f} (x_{i, \texttt{low}}), \tilde{f}(x_{i, \texttt{high}})]$ contains $y$ for each $i$. This implies that $x = f^{-1}(y)$ is in the interval 
    $[f^{-1}(\tilde{f} (x_{t, \texttt{low}})), f^{-1}(\tilde{f}(x_{t, \texttt{high}}))]$, since $f$ is increasing. 
    Moreover, since $f^{-1}$ is $L$-Lipschitz and $\tilde{f}$ is accurate to within $\eps$, the endpoints of this interval are within $\eps L$ of $x_{t, \texttt{low}}$ and $x_{t, \texttt{high}}$, so we know that the above interval is contained in 
    \[
        [x_{t, \texttt{low}} - \eps L , x_{t, \texttt{high}} + \eps L] .
    \]
    So $x$ is contained in this interval, but $x_{t, \texttt{mid}}$ is the midpoint of this interval, so the maximum possible distance between $x$ and $x_{t, \texttt{mid}}$ is half the length of the interval, which is $(b-a)2^{t+1} + \eps L$.
\end{proof}

\subsection{Proof of \autoref{uniformtonormal}}

\begin{proof}
  Applying \autoref{inversenormalcdfapprox} with $\eps = \eps_1$ (to be specified later), we approximate the inverse CDF of the normal distribution on $[\Phi(-\ln(1/\eps_1^2)), \Phi(\ln(1/\eps_1^2))]$. On the intervals $[0, \Phi(-\ln(1/\eps_1^2))]$ and $[\Phi(\ln(1/\eps_1^2)), 1]$, we set the output of the network to $-\ln(1/\eps_1^2)$ and $\ln(1/\eps_1^2)$ respectively (using step function activations to test if the input is in this range). Finally, we append a gadget computing the function $x \mapsto \max(-\ln(1/\eps_1^2), \min(x, \ln(1/\eps_1^2)))$, so that the output of our network $f$ lies in the range $[-\ln(1/\eps_1^2), \ln(1/\eps_1^2)]$. We now look to lower bound the Wasserstein for this generative network $f$. 
  \begin{align*}
        W(f \# U([0,1]), \cN)
        &= \inf_{\pi \in \Pi} |x - y| d\pi(x,y) .\\
        \intertext{We consider a coupling between $f \# U([0,1])$ and $\cN$ with pairs of the form $(f(x), \Phi^{-1}(x))$, where $x \sim U([0,1])$:}
        &\le \int_0^1 |f(x) - \Phi^{-1}(x)| dx. \\
        \intertext{We now split this integral into three parts}
        &= \int_{A} |f(x) - \Phi^{-1}(x)| dx + \int_{I \setminus A} |f(x) - \Phi^{-1}(x)| dx + \int_{[0,1] \setminus (I \cup A)} |f(x) - \Phi^{-1}(x)| dx \\
        &\le m(A) \cdot 2\ln(1/\eps_1^2)  + 1 \cdot \eps_1 + 2\int_{\Phi(\ln(1/\eps_1^2))}^1 |f(x) - \Phi^{-1}(x)| dx. \\
        \intertext{Rewriting the integral on the tails,}
        &= m(A) \cdot 2\ln(1/\eps_1^2)  + 1 \cdot \eps_1 + 2\int_{\ln(1/\eps_1^2)}^\infty 1-\Phi(x) dx, \\
        \intertext{and since the normal CDF has exponentially small tails}
        &= m(A) \cdot 2\ln(1/\eps_1^2)  + O(\eps_1). \\
        \intertext{Now, choosing $m(A)$ sufficiently small,}
        &= O(\eps_1), \\
  \end{align*}
  and we can choose $\eps_1$ sufficiently small so that this is under $\eps$. Since $\eps_1$ is linear in $\eps$, the construction still uses $\polylog(1/\eps)$ nodes.
\end{proof}

\subsection{Proof of \autoref{piecewiseaffinecdf}}

\begin{proof}
    Let $[0=a_0, a_1], [a_1, a_2], \ldots, [a_{N_A-1}, a_{N_A}=1]$ be the intervals on which $f|_{[0,1]}$ is affine. The distribution given by $f\# U([0,1)$ is a mixture of $N_A$ distributions
    \[
        f \# U([0,1]) = \sum_{i=0}^{N_A-1} \frac{1}{a_{i+1} - a_i} f \# U([a_i, a_{i+1}]),
    \]
    where the $f \# U([a_i, a_{i+1}])$ is either a uniform distribution on the interval $[f(a_i), f(a_{i+1})]$ (or $[f(a_{i+1}), f(a_{i})]$), or if $f(a_i) = f(a_{i+1})$, it is a point distribution on $f(a_i)$. Since these distributions have CDFs which are nonlinear only at $f(a_i)$ values, the CDF of $f \# U([0,1])$ (which is the weighted sum of CDFs of these simple distributions), is piecewise affine with nonlinearities at $f(a_i)$. In other words, it is piecewise affine in $N_A + 2$ pieces.
\end{proof}

\subsection{Proof of \autoref{affineapproxbound}}

\begin{proof}
    Define $g$ to be the function
    \[
        g(x) = \max(0, \min(1, f(x))).
    \]
    Since $f$ has $N_A$ affine pieces, and each of these can yield at most $3$ pieces in $g$, $g$ has $3N_A$ affine pieces at most.
    By \citet[Theorem 7]{DBLP:journals/corr/SafranS16}, since $\Phi$ has second derivative bounded away from 0 on an interval, we get 
    \[
        \int_a^b |\Phi(x) - g(x)|^2 dx \ge \frac{k}{(3N_A)^4} = \frac{K}{N_A^4}
    \]
    for some constant $K$. Since $|\Phi(x) - g(x)| \le 1$, we have
    \[
        \int_a^b |\Phi(x) - f(x)| dx \ge \int_a^b |\Phi(x) - g(x)| dx \ge \frac{K}{N_A^4} .
    \]
\end{proof}

\subsection{Proof of \autoref{berryesseennonuniform}}

\begin{proof}
    Each of the uniform variables has mean $1/2$ and variance $1/12$ so subtract $n/2$ and multiply the sum by $1/\sqrt{12}$ to normalize. 
    The nonuniform version of the Berry-Esseen theorem \citep{Onthenon96:online} tells us that there is a constant $C$ such that the difference in CDF between this and the normal CDF at $t$ is no more than $\frac{C}{\sqrt{n} (1 + |t|^3)}$. 
    Since the integral of $1/(1+|t|^3)$ converges, the integral of this difference over all $\R$ is bounded by $O(\frac{1}{\sqrt{n}})$ and by \autoref{cdfdifference}, this gives the Wasserstein distance bound.
\end{proof}

\section{Additional Lemma}

\begin{lemma} \label{compositionlemma}
    If we have $A,B,C$ subsets of Euclidean spaces and functions $f, \tilde{f}: A \to B$ and $g, \tilde{g}: B \to C$ 
    such that
    \begin{itemize}
        \item
        For all $x \in A$, $|\tilde{f}(x) - f(x)| < \frac{\eps}{2L_g}$ (where $L_g$ is a Lipschitz constant of $g$)
        \item
        For all $y \in B$, $|\tilde{g}(y) - g(y)| < \frac{\eps}{2}$
    \end{itemize}  
    then $| (\tilde{f} \circ \tilde{g})(x) - (f \circ g)(x)| \le \eps$ for all $x \in A$.
\end{lemma}

\begin{proof}
    If $|\tilde{f}(x) - f(x)| < \eps_1 = \frac{\eps}{2L_g}$ 
    and $|\tilde{g} - g| < \eps_2 = \frac{\eps}{2}$, then applying the triangle inequality gives us
    \begin{align*}
        |(\tilde{g} \circ \tilde{f})(x) - (g \circ f)(x)| 
        &= |(\tilde{g} \circ \tilde{f})(x) - (g \circ \tilde{f})(x) + (g \circ \tilde{f})(x) - (g \circ f)(x)| \\
        &\le |(\tilde{g} \circ \tilde{f})(x) - (g \circ \tilde{f})(x)| + |(g \circ \tilde{f})(x) - (g \circ f)(x)| \\
        &\le \eps_2 + L_g \eps_1 \\
        &= \eps .
    \end{align*}
\end{proof}

\end{document}